\documentclass{article}
\usepackage{times}
\usepackage{graphicx}
\usepackage{amsmath,amssymb,amsbsy,amsfonts,amsthm}
\usepackage{tikz}
\usetikzlibrary{calc,math,shapes,patterns,arrows,decorations.markings}
\usepackage{biblatex}
\AtEveryBibitem{
	\clearlist{language}
	\clearfield{doi}
	\clearfield{issn}
	\clearfield{isbn}
	\clearfield{urlyear}
	\ifentrytype{misc}{}{\clearfield{url}}
}
\usepackage[text={17.5cm,25cm},centering]{geometry}
\newcommand{\dd}{\mathbf{d}}

\newcommand{\expectation}{\mathbb{E}}
\newcommand{\carac}[1]{\mathbb{I}[{#1}]}
\newcommand{\mbeta}{\mathbf{Boojum}}
\newcommand\inner[2]{\langle{#1},{#2}\rangle}

\newtheorem{definition}{Definition}
\newtheorem{lemma}{Lemma}
\newtheorem{theorem}{Theorem}

\title{A conjugate prior for the Dirichlet distribution}
\author{Jean-Marc Andreoli\\Naverlabs Europe}
\date{October 2014; last update October 2018}
\begin{document}
\maketitle
\begin{abstract}
This note investigates a conjugate class for the Dirichlet distribution class in the exponential family.
\end{abstract}
\section{Basic definitions}
The exponential family of distributions is characterised by the existence of a systematic procedure to produce priors within that same family (see, e.g., Proposition 3.3.13 in~\cite{robert_bayesian_2007}). For example, the class of Dirichlet distributions can be obtained by application of that procedure to the class of multinomial distributions, an essential class in the exponential family (see, e.g., Theorem 4.3 in~\cite{sethuraman_constructive_1994}, and the application of Dirichlet priors to document clustering in~\cite{blei_latent_2003}). Now, applying the same procedure to the class of Dirichlet distributions itself, we obtain a new class of distributions, still in the exponential family, which we denote here as $\mbeta$ for lack of a better name.
\begin{definition}
The $\mbeta(m,r)$ distribution, with parameter $m\in\mathbb{R}$ (shape) and $r\in\mathbb{R}^K$ (rate vector), is defined for $x\in\mathbb{R}_+^K$ by
\begin{eqnarray*}
\mbeta(x;m,r) & \triangleq & \frac{1}{Z(m,r)}\mathcal{B}(x)^{-m}\exp-\sum_kr_kx_k\\
Z(m,r) & \triangleq & \int_x \mathcal{B}(x)^{-m}\exp-\sum_kr_kx_k\dd{x}
\end{eqnarray*}
where $\mathcal{B}$ denotes the muti-variate Beta function $\mathcal{B}(x)\triangleq\frac{\prod_k\Gamma(x_k)}{\Gamma(\sum_kx_k)}$, which is also the normalising constant of the Dirichlet distribution of parameter $x$.
\end{definition}
Alternatively, the distribution can be defined over $\mathbb{R}_+{\times}\mathcal{T}_K$ where $\mathcal{T}_K\triangleq\{t\in\mathbb{R}_+^K|\sum_{k\in K}t_k=1\}$ is the $\mathbb{R}_+^K$-simplex. We define the homeomorphism from $x\in\mathbb{R}_+^K$ into $s,t\in\mathbb{R}_+{\times}\mathcal{T}_K$ as follows:
\[
\begin{array}{l@{\hspace{1cm}}l@{\hspace{1cm}}l}
\text{direct:} & s = \sum_kx_k & \forall k\;t_k=\frac{x_k}{s}\\
\text{reverse:} & \multicolumn{2}{c}{\forall k\;x_k=st_k}
\end{array}
\]
The computation of the Jacobian yields $\dd{x}=s^{K-1}\dd{s}\dd{t}$, where $\dd{t}$ denotes the measure on $\mathcal{T}_K$ obtained by projection of the simplex along any arbitrarily chosen axis. Hence the distribution in the alternate space $\mathbb{R}_+\times\mathcal{T}_K$ is given by
\begin{eqnarray*}
p(t|s) & = & \frac{1}{\bar{Z}(s)}\mathcal{B}(st)^{-m}\exp-s\inner{r}{t}
\hspace{2cm}
\bar{Z}(s) \triangleq \int_{t\in\mathcal{T}}\mathcal{B}(st)^{-m}\exp-s\inner{r}{t}\dd{t} \\
p(s) & = & \frac{1}{Z(m,r)}\bar{Z}(s)s^{K-1}
\hspace{2cm}
Z(m,r) = \int_s \bar{Z}(s)s^{K-1}\dd{s}
\end{eqnarray*}
Here, $\inner{.}{.}$ denotes the scalar product in $\mathbb{R}^K$. To be strict, we should write $\bar{Z}(s;m,r)$ for the normalising constant of the conditional, but we omit the dependency on $m,r$ for the sake of clarity.
\section{Study of convergence}
We seek to determine the values of the parameter $m,r$ which lead to a proper distribution, i.e. where the normalising constant $Z(m,r)$ is finite. The main result, expressed by the following theorem, is informally summarised in Figure~\ref{fig:convergence} when $K{=}2$.
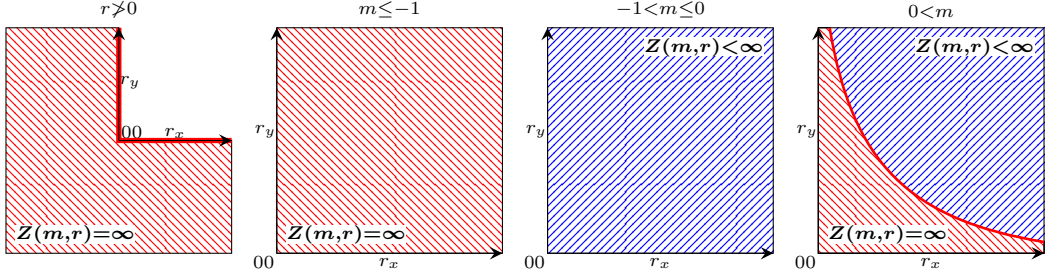
\begin{figure}
\begin{center}
\begin{tikzpicture}[scale=3]
\begin{scope}[
myptr/.style={decoration={markings,mark=at position 1 with {\arrow[scale=1.5,>=stealth]{>}}},postaction={decorate}}
]
\node[label={[above,inner sep=0pt]:$\scriptstyle{r\not>0}$}] at (.5,1) {};
\node[label={[inner sep=0pt]45:$\scriptstyle{00}$},inner sep=0pt] at (.5,.5) {};
\node[label={[above,inner sep=0pt]:$\scriptstyle{r_x}$},inner sep=0pt] at (.75,.5) {};
\node[label={[right,inner sep=0pt]:$\scriptstyle{r_y}$},inner sep=0pt] at (.5,.75) {};
\draw[pattern=north west lines,pattern color=red] (0,0) -- (1,0) -- (1,.5) -- (.5,.5) -- (.5,1) -- (0,1) -- cycle;
\node[label={[inner sep=0pt,fill=white]45:$\boldsymbol{\scriptstyle{Z(m,r)=\infty}}$}] at (0,0) {};
\draw [red,line width=2pt] (.5,1) -- (.5,.5) -- (1,.5);
\draw [myptr] (.5,.5) -- (1,.5);\draw [myptr] (.5,.5) -- (.5,1);
\begin{scope}[shift={(1.2,0)}]
\node[label={[above,inner sep=0pt]:$\scriptstyle{m\leq-1}$}] at (.5,1) {};
\node[label={[inner sep=0pt]225:$\scriptstyle{00}$},inner sep=0pt] at (0,0) {};
\node[label={[below,inner sep=5pt]:$\scriptstyle{r_x}$}] at (.5,0) {};
\node[label={[left,inner sep=0pt]:$\scriptstyle{r_y}$}] at (0,.5) {};
\draw[pattern=north west lines,pattern color=red] (0,0) rectangle (1,1);
\node[label={[inner sep=0pt,fill=white]45:$\boldsymbol{\scriptstyle{Z(m,r)=\infty}}$}] at (0,0) {};
\draw [myptr] (0,0) -- (1,0);\draw [myptr] (0,0) -- (0,1);
\end{scope}
\begin{scope}[shift={(2.4,0)}]
\node[label={[above,inner sep=0pt]:$\scriptstyle{-1<m\leq0}$}] at (.5,1) {};
\node[label={[inner sep=0pt]225:$\scriptstyle{00}$},inner sep=0pt] at (0,0) {};
\node[label={[below,inner sep=5pt]:$\scriptstyle{r_x}$}] at (.5,0) {};
\node[label={[left,inner sep=0pt]:$\scriptstyle{r_y}$}] at (0,.5) {};
\draw[pattern=north east lines,pattern color=blue] (0,0) rectangle (1,1);
\node[label={[inner sep=0pt,fill=white]225:$\boldsymbol{\scriptstyle{Z(m,r)<\infty}}$}] at (1,1) {};
\draw [myptr] (0,0) -- (1,0);\draw [myptr] (0,0) -- (0,1);
\end{scope}
\begin{scope}[shift={(3.6,0)}]
\node[label={[above,inner sep=0pt]:$\scriptstyle{0<m}$}] at (.5,1) {};
\node[label={[inner sep=0pt]225:$\scriptstyle{00}$},inner sep=0pt] at (0,0) {};
\node[label={[below,inner sep=5pt]:$\scriptstyle{r_x}$}] at (.5,0) {};
\node[label={[left,inner sep=0pt]:$\scriptstyle{r_y}$}] at (0,.5) {};
\draw[pattern=north east lines,pattern color=blue] (0,0) rectangle (1,1);
\draw[fill=white] (0,0) -- (0,1) -- (.05,1) .. controls (.15,.4) and (.4,.15) .. (1,.05) -- (1,0) -- cycle;
\draw[pattern=north west lines,pattern color=red] (0,0) -- (0,1) -- (.05,1) .. controls (.15,.4) and (.4,.15) .. (1,.05) -- (1,0) -- cycle;
\draw[red,line width=1pt] (.05,1) .. controls (.15,.4) and (.4,.15) .. (1,.05);
\node[label={[inner sep=0pt,fill=white]45:$\boldsymbol{\scriptstyle{Z(m,r)=\infty}}$}] at (0,0) {};
\node[label={[inner sep=0pt,fill=white]225:$\boldsymbol{\scriptstyle{Z(m,r)<\infty}}$}] at (1,1) {};
\draw [myptr] (0,0) -- (1,0);\draw [myptr] (0,0) -- (0,1);
\end{scope}
\end{scope}
\end{tikzpicture}
\end{center}
\caption{\label{fig:convergence}Summary of the convergence results. When $0<m$, the boundary corresponds to $\sum_k\exp-\frac{r_k}{m}=1$.}
\end{figure}
\begin{theorem}
\label{thm:properness}
Let $m\in\mathbb{R}$ and $r\in\mathbb{R}^K$. The distribution $\mbeta(m,r)$ is proper, i.e. $Z(m,r)$ is finite, if and only if
\[
\forall k\in K\;r_k>0
\hspace{.5cm}\text{and}\hspace{.5cm}
m>-1
\hspace{.5cm}\text{and}\hspace{.5cm}
\left(m\leq0\hspace{.3cm}\text{or}\hspace{.3cm}\sum_k\exp-\frac{r_k}{m}<1\right)
\]
\end{theorem}
The rest of this section is devoted to the proof of this result, which summarises Lemmas~\ref{lem:m-below_1-or-r-nonpositive}, \ref{lem:m-zero}, \ref{lem:m-negative-above_1}, \ref{lem:m-positive-T-not1}, \ref{lem:m-positive-T-1}.
\begin{lemma}
\label{lem:z-bar}
For any given $s>0$, the integral $\bar{Z}(s)$ is finite if and only if $m>-1$.
\end{lemma}
\begin{proof}
For a given $s$, we have, using the definition of $\bar{Z}$ and the property $\Gamma(x){=}\frac{1}{x}\Gamma(1+x)$
\[
\bar{Z}(s) = \int_{t\in\mathcal{T}}\left(\frac{\prod_k\Gamma(st_k)}{\Gamma(s)}\right)^{-m}\exp-s\inner{r}{t}\dd{t}
= \int_{t\in\mathcal{T}}\prod_kt_k^m\underline{\left(\frac{\prod_k\Gamma(1+st_k)}{s^K\Gamma(s)}\right)^{-m}\exp-s\inner{r}{t}}\dd{t}
\]
The underlined term is a continuous function of $t$ over the compact $\mathcal{T}$, hence both lower and upper bounded. Hence, for some strictly positive constants $u_s^\bot,u_s^\top$ which depend on $s$, we have
\[
u_s^\bot\int_{t\in\mathcal{T}}\prod_kt_k^m\dd{t}\leq \bar{Z}(s)\leq u_s^\top\int_{t\in\mathcal{T}}\prod_kt_k^m\dd{t}
\]
Hence, $\bar{Z}(s)$ is finite if and only if so is $\int_{t\in\mathcal{T}}\prod_{k\in K}t_k^m\dd{t}$. The latter is the normalising constant of the balanced Dirichlet distribution with parameter $m+1$, which is proper if and only if $m+1>0$.
\end{proof}
\subsection{The case $m\leq-1$ or $r\not>0$}
\begin{lemma}
\label{lem:m-below_1-or-r-nonpositive}
If $m\leq-1$ or $r\not>0$, then $Z(m,r)$ is infinite.
\end{lemma}
\begin{proof}
When $m\leq-1$, consider $Z(m,r)$ as an integral on space $\mathbb{R}_+\times\mathcal{T}_K$. By Lemma~\ref{lem:z-bar}, the integrand $\bar{Z}(s)$ is infinite for all $s$, hence, obviously, so is the integral $Z(m,r)$. Now, let's assume $m>-1$ and $r\not>0$, i.e. $r_k\leq0$ for some $k\in K$. Consider $Z(m,r)$ as an integral on space $\mathbb{R}_+^K$. We have, isolating $x_k$ in the integrand,
\[
Z(m,r) = \int_{x\in\mathbb{R}_+^{K\setminus\{k\}}}F(\sum_{h\not=k}x_h)\prod_{h\not=k}\Gamma(x_h)^{-m}\exp-r_hx_h\dd{x_h}
\hspace{.5cm}\text{where}\hspace{.5cm}
F(z) \triangleq \int_{x\in\mathbb{R}_+}\left(\frac{\Gamma(x)}{\Gamma(z+x)}\right)^{-m}\exp-r_kx\dd{x}
\]
Let $\Delta\triangleq(0,1]$ if $m\leq0$ and $\Delta\triangleq[1,\infty)$ if $0<m$.
Recall that $\Gamma$ is increasing on $[2,\infty)$. Hence, for any $z\in\Delta$ and $x\geq2$:
\begin{itemize}
\item
If $m\leq0$ then $z\leq1$ hence $\Gamma(z+x)\leq\Gamma(1+x)=x\Gamma(x)$ hence $\left(\frac{\Gamma(x)}{\Gamma(z+x)}\right)^{-m}\geq x^m$
\item
If $0<m$ then $z\geq1$ hence $\Gamma(z+x)\geq\Gamma(1+x)=x\Gamma(x)$ hence $\left(\frac{\Gamma(x)}{\Gamma(z+x)}\right)^{-m}\geq x^m$
\end{itemize}
In both cases, by integration, we have for any $z\in\Delta$
\[
F(z)\geq\int_{x\geq2}x^m\exp-r_kx\dd{x}\geq\int_{x\geq2}x^m\dd{x}=\infty
\]
Hence $F(\sum_{h\not=k}x_h)$ is infinite on a non null subset of $\mathbb{R}_+^{K\setminus\{k\}}$, hence $Z(m,r)$ is infinite.
\end{proof}
\subsection{The case $r>0$ and $-1<m$}
We now assume $r>0$ and $-1<m$. One case can easily be treated:
\begin{lemma}
\label{lem:m-zero}
If $r>0$ and $m=0$ then $\mbeta(0,r)$ is a multivariate exponential distribution and $Z(0,r)=\prod_kr_k^{-1}$.
\end{lemma}
\begin{proof}
Simply observe that $\mbeta(0,r)$ is the product of independent exponentials, each with rate $r_k$ for $k{=}1{:}K$.
\end{proof}
When $m\not=0$, the finiteness of $Z(m,r)$ depends on the behaviour of $\bar{Z}$ (which is finite by Lemma~\ref{lem:z-bar}) near $0$ and near $\infty$. One side (depending on the sign of $m$) can easily be treated.
\begin{lemma}
\label{lem:z-restricted}
If $r>0$ and $-1<m\not=0$, then $Z_\Delta(m,r)$ is finite, where
\begin{eqnarray}
\label{eqn:z-delta}
Z_\Delta(m,r) & \triangleq & \int_{s\in\Delta}\bar{Z}(s)s^{K-1}\dd{s}
\hspace{1cm}\text{and}\hspace{1cm}
\Delta\triangleq\left\{\begin{array}{l@{\hspace{.25cm}\text{if}\hspace{.25cm}}l} [1,\infty) & m<0 \\ {(0,1]} & m>0 \end{array}\right.
\end{eqnarray}
Hence, $Z(m,r)$ is finite if and only if so is $Z_{\mathbb{R}_+\setminus\Delta}$, since $Z(m,r)=Z_\Delta+Z_{\mathbb{R}_+\setminus\Delta}$.
\end{lemma}
\begin{proof}
For a fixed $t\in\mathcal{T}$, by derivation, we have
\[
\frac{\dd{\log\mathcal{B}(st)}}{\dd{s}} =
\frac{\dd}{\dd{s}}(\sum_k\log\Gamma(st_k)-\log\Gamma(s)) =
\sum_kt_k\Psi(st_k)-\Psi(s)<\sum_kt_k\Psi(s)-\Psi(s)=0
\]
Here $\Psi$ is the derivative of the $\log\Gamma$ function (a.k.a. digamma function), which is increasing~\cite{mortici_new_2011}, hence $\Psi(st_k)<\Psi(s)$. Therefore, $\log\mathcal{B}(st)$ is a decreasing function of $s$, hence so is $\mathcal{B}(st)$. Hence $\mathcal{B}(st)^{-m}$ is $m$-increasing in $s$ (meaning increasing when $m>0$ and decreasing when $m<0$).

let $r^*\triangleq\min_kr_k$ when $m<0$ and $r^*\triangleq\max_kr_k$ when $m>0$. By construction $\exp-s\inner{r-r^*}{t}$ is also $m$-increasing in $s$ (or constant if all the components of $r$ are equal), hence $\mathcal{B}(st)^{-m}\exp-s\inner{r-r^*}{t}$ is $m$-increasing in $s$. By integration over $t$, we get that $F(s)\triangleq\bar{Z}(s;m,r-r^*)$, which is finite by Lemma~\ref{lem:z-bar}, is $m$-increasing in $s$.

Observe that $\exp-s\inner{r-r^*}{t}=\exp sr^*\exp-s\inner{r}{t}$, since $\sum_kt_k=1$. Hence $F(s)=\bar{Z}(s)\exp sr^*$ and
\[
Z_\Delta = \int_{s\in\Delta} F(s)s^{K-1}\exp-r^*s\dd{s}
\]
and, since $F$ is $m$-increasing, we have (the split value 1 is arbitrary)
\[
\begin{array}{r@{\hspace{.5cm}}l@{=}l}
\text{when }m<0\text{ ($F$ is decreasing):} & Z_{[1,\infty)} &
\int_{s\geq1} F(s)s^{K-1}\exp-r^*s\dd{s}\leq F(1)\int_{s\geq1}s^{K-1}\exp-r^* s\dd{s} <\infty\\
\text{when }m>0\text{ ($F$ is increasing):} & Z_{(0,1]} &
\int_{s\leq1} F(s)s^{K-1}\exp-r^*s\dd{s}\leq F(1)\int_{s\leq1}s^{K-1}\exp-r^* s\dd{s} <\infty
\end{array}
\]
Hence, $Z(m,r)$ is finite if and only if so is $Z_{\Delta}(m,r)$ where $\Delta=(0,1]$ when $m<0$ and $\Delta=[1,\infty)$ when $m>0$.
\end{proof}
\begin{figure}
\begin{center}
\includegraphics[scale=.4]{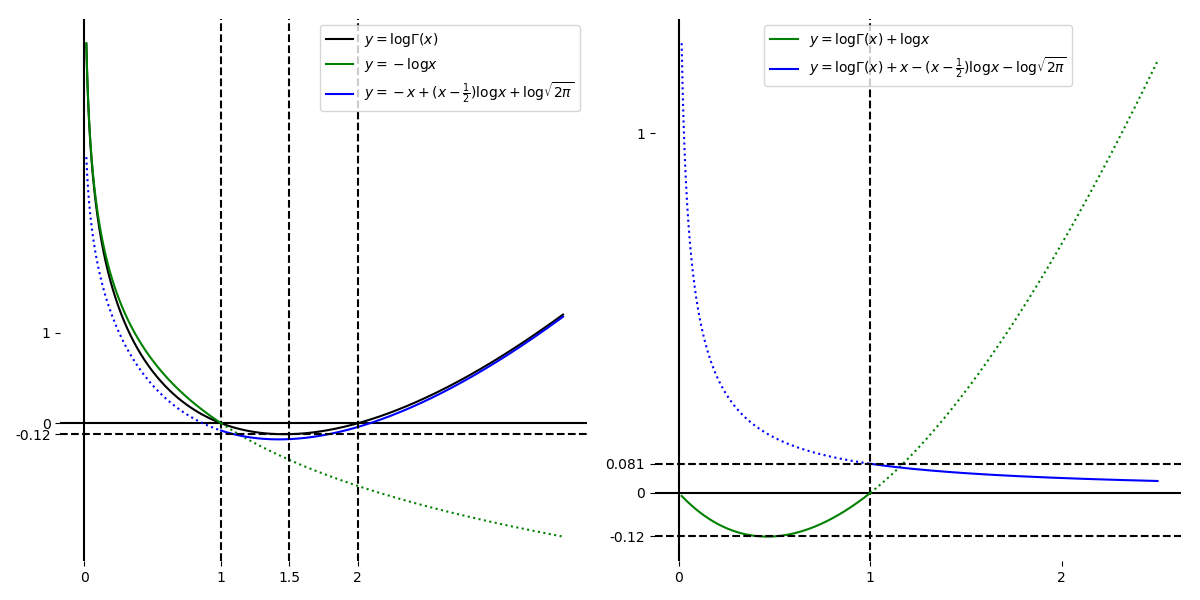}
\end{center}
\caption{\label{fig:approx}Left: curve of the $\log\Gamma$ function and its two approximations. Right: detail of the approximation gaps.}
\end{figure}
Thus, we need only study the behaviour of $\bar{Z}$ near $0$ when $-1<m<0$ and near $\infty$ when $m>0$. For that, we look for an approximation of $\log\mathcal{B}(st)$, which we obtain from an approximation of the $\log\Gamma$ function, actually one near $0$ and a different one near $\infty$, as illustrated in Figure~\ref{fig:approx}. Expanding the definition of $\bar{Z}$ in Equation~(\ref{eqn:z-delta}) gives:
\begin{eqnarray}
\label{eqn:z-delta-exp}
Z_\Delta & = & \int_{s\in\Delta}s^{K-1}\int_{t\in\mathcal{T}}\exp(-m\log\mathcal{B}(st)-s\inner{r}{t})\dd{t}\dd{s}
\end{eqnarray}
\subsection{The case $r>0$ and $-1<m<0$}
\begin{lemma}
\label{lem:m-negative-above_1}
If $r>0$ and $-1<m<0$, then $Z(m,r)$ is finite.
\end{lemma}
\begin{proof}
Define $h(z)\triangleq\log\Gamma(z)+\log z$. Thus we have
\[
\left.
\begin{array}{rcl}
\log\Gamma(s) & = & -\log s+h(s)\\
\sum_k\log\Gamma(st_k) & = & \sum_k-\log st_k+h(st_k)
\end{array}
\right\}
\Rightarrow\;
-\log\mathcal{B}(st)=(K-1)\log s+\sum_k\log t_k+\underbrace{h(s)-\sum_kh(st_k)}_{H(s,t)}
\]
Hence, reporting in Equation~(\ref{eqn:z-delta-exp}), we get
\begin{eqnarray*}
Z_{(0,1]} & = &
\int_{s\leq1}s^{K-1}\int_{t\in\mathcal{T}}\exp\left(m((K-1)\log s+\sum_k\log t_k+H(s,t))-s\inner{r}{t}\right)\dd{t}\dd{s}\\
& = & \int_{s\leq1}s^{(1+m)(K-1)}\int_{t\in\mathcal{T}}\prod_kt_k^m\exp(-mH(s,t)-s\inner{r}{t})\dd{t}\dd{s}\\
\end{eqnarray*}
By construction, $h$ is bounded on $(0,1]$, hence $H(s,t)$ is bounded for $s,t\in(0,1]{\times}\mathcal{T}$, and so is $s\inner{r}{t}$. Hence, we have for some strictly positive constant $U$
\[
Z_{(0,1]}\leq U\int_{s\leq1}s^{(1+m)(K-1)}\int_{t\in\mathcal{T}}\prod_kt_k^m\dd{t}\dd{s}=U\mathcal{B}(1+m)\int_{s\leq1}s^{(1+m)(K-1)}\dd{s}<\infty
\]
Hence $Z_{(0,1]}$ is finite, hence so is $Z(m,r)$ by Lemma~\ref{lem:z-restricted}.
\end{proof}
\subsection{The case $r>0$ and $0<m$}
\begin{lemma}
\label{lem:m-positive-T-not1}
If $r>0$ and $0<m$, then $Z(m,r)$ is finite if $T<1$ and infinite if $T>1$, where $T\triangleq\sum_k\exp-\frac{r_k}{m}$. 
\end{lemma}
\begin{proof}
Define $h(z)\triangleq\log\Gamma(z)+z-(z-\frac{1}{2})\log z$. Thus, $h$ denotes the difference between $\log\Gamma$ and its Stirling's approximation (up to a constant). We have, for any $t\in\mathcal{T}$, using $\sum_kt_k=1$,
\[
\begin{array}{l}
\left.
\begin{array}{rcl}
\log\Gamma(s) & = & -s+(s-\frac{1}{2})\log s+h(s)\\
\sum_k\log\Gamma(st_k) & = & \sum_k-st_k+(st_k-\frac{1}{2})\log st_k+h(st_k)
\end{array}
\right\}
\Rightarrow\\ \\
-\log\mathcal{B}(st)=\frac{K-1}{2}\log s-s\underbrace{\sum_kt_k\log t_k}_{J(t)}+\underbrace{h(s)-\sum_kh(st_k)+\frac{1}{2}\sum_k\log t_k}_{H(s,t)}
\end{array}
\]
Observe that the terms in $s\log s$ cancelled out. Hence, reporting in Equation~(\ref{eqn:z-delta-exp}), we get
\begin{eqnarray*}
Z_{[1,\infty)} & = &
\int_{s\geq1}s^{K-1}\int_{t\in\mathcal{T}}\exp\left(m(\frac{K-1}{2}\log s-sJ(t)+H(s,t))-s\inner{r}{t}\right)\dd{t}\dd{s}\\
& = & \int_{s\geq1}s^{(1+\frac{m}{2})(K-1)}\int_{t\in\mathcal{T}}\exp m\left(-s(J(t)+\frac{\inner{r}{t}}{m})+H(s,t)\right)\dd{t}\dd{s}
\end{eqnarray*}
We first seek bounds for the term $H(s,t)$. Function $h$ is decreasing and lower bounded by $\log\sqrt{2\pi}$, hence positive~\cite{mortici_new_2011}. Therefore, for any $s\geq1$:
\begin{itemize}
\item
For any $t\in\mathcal{T}$, since $h(st_k)>0$ and $t_k\leq1$, hence $\log t_k\leq0$, we have $H(s,t)<h(s)\leq h(1)$.
\item
On the other hand, $h$ is not upper bounded near $0$, hence for any given $s$, the function $H(s,.)$ is not lower bounded near the border of $\mathcal{T}$. However, choose $\tau\in\mathcal{T}$ and for any $0\leq\alpha\leq1$ consider the sub-domain of $\mathcal{T}$ defined by
\begin{eqnarray*}
\mathcal{T}_\alpha & \triangleq & \{t\in\mathcal{T}|\forall k\in K\;t_k\geq\alpha\tau_k\}
\end{eqnarray*}
\begin{figure}
\begin{center}
\begin{tikzpicture}[scale=2]
\coordinate (A) at (0,0);
\coordinate (B) at (2,0);
\coordinate (C) at (1,{sqrt(3)});
\coordinate (T) at (1.4,.35);
\coordinate (Ax) at ($ (A)!.3!(T) $);
\coordinate (Bx) at ($ (B)!.3!(T) $);
\coordinate (Cx) at ($ (C)!.3!(T) $);
\draw (A) -- (B) -- (C) -- cycle;
\draw[dotted] (A) -- (T);
\draw[dotted] (B) -- (T);
\draw[dotted] (C) -- (T);
\draw[blue] (Ax) -- (Bx) -- (Cx) -- cycle;

\node [label={[inner sep=0pt,xshift=1.5mm]120:$\scriptstyle{\mathcal{T}_0=\mathcal{T}}$}] at ($ (A)!.5!(C) $) {};
\node [label={[blue,inner sep=0pt,xshift=1.5mm]120:$\scriptstyle{\mathcal{T}_\alpha}$}] at ($ (Ax)!.5!(Cx) $) {};
\node [fill=red,isosceles triangle,rotate=90,label={[red,inner sep=0pt]85:$\scriptstyle{\mathcal{T}_1=\{\tau\}}$},inner sep=0pt,minimum width=1mm] at (T) {};
\end{tikzpicture}
\end{center}
\caption{\label{fig:subdomain}An illustration of the sub-domain $\mathcal{T}_\alpha$ of the simplex $\mathcal{T}$ used for lower bounding $Z(m,r)$ when $m>0$. The idea is to avoid the border of $\mathcal{T}$ (when $\alpha>0$) while still being non-null (when $\alpha<1$).}
\end{figure}
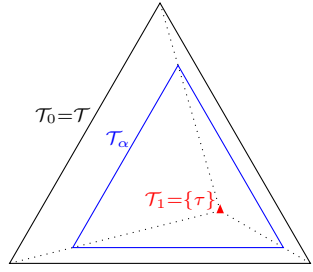
This compact convex set, illustrated in Figure~\ref{fig:subdomain}, contains $\tau$ and has a non null measure in $\mathcal{T}$, whenever $\alpha<1$ (for $\alpha=1$ it degenerates into the singleton $\{\tau\}$). For $\alpha=0$ we have $\mathcal{T}_0=\mathcal{T}$ over which $H(s,.)$ is not lower bounded, but whenever $\alpha>0$ we get a lower bound over $\mathcal{T}_\alpha$, since $-h$ and $\log$ are increasing and $\forall t\in\mathcal{T}_\alpha\;t\geq\alpha\tau$:
\[
H(s,t)\geq h(s)-\sum_kh(s\alpha\tau_k)+\frac{1}{2}\sum_k\log(\alpha\tau_k)\geq-\sum_kh(\alpha\tau_k)+\frac{1}{2}\sum_k\log(\alpha\tau_k)
\]
\end{itemize}
Neither the upper bound nor the lower bound of $H(s,t)$ are dependent on $s,t$, hence, we have for some strictly positive constants $U_\alpha^\bot,U^\top$, for all $s,t\in[1,\infty){\times}\mathcal{T}_\alpha$:
\[
U_\alpha^\bot\;\;\leq\exp(mH(s,t))\leq\;\; U^\top
\]
Let's introduce the short-hands $M\triangleq(1+\frac{m}{2})(K-1)$ and
\begin{eqnarray}
\label{eqn:F}
F_\alpha(s)& \triangleq&\int_{t\in\mathcal{T}_\alpha}\exp-ms(J(t)+\frac{\inner{r}{t}}{m})\dd{t}
\end{eqnarray}
Observe that $M>0$. Using the bounds on $\exp(mH)$, we get:
\begin{equation}
\label{eqn:Z-bound}
U_\alpha^\bot\int_{s\geq1}s^MF_\alpha(s)\dd{s}
\;\;\leq Z_{[1,\infty)}\leq\;\;
U^\top\int_{s\geq1}s^MF_0(s)\dd{s}
\end{equation}
We now seek bounds for $F_\alpha$. We use a specific choice of $\tau$:
\[
\tau_k \triangleq \frac{1}{T}\exp-\frac{r_k}{m}
\hspace{1cm}\text{where}\hspace{1cm}
T \triangleq \sum_k\exp-\frac{r_k}{m}
\]
It is then easy to show that, for all $t\in\mathcal{T}$,
\[
J(t)+\frac{\inner{r}{t}}{m}=D(t)-\log T
\hspace{1cm}\text{where}\hspace{1cm}
D(t) \triangleq \sum_k t_k\log\frac{t_k}{\tau_k}
\]
Observe that $D(t)$ is the information divergence between $t$ and $\tau$ viewed as discrete distributions over $K$. It is a convex function of $t$ which has minimum $0$ reached at $\tau$ and maximum $\mu_\alpha$ over $\mathcal{T}_\alpha$, reached at one of its corners. Hence, reporting in Equation~(\ref{eqn:F}) we get, for any $s\geq1$
\[
\mathcal{A}_\alpha\exp ms(\log T-\mu_\alpha)
\;\;\leq F_\alpha(s)\leq\;\;
\mathcal{A}_\alpha\exp ms\log T
\hspace{1cm}\text{where}\hspace{1cm} \mathcal{A}_\alpha\triangleq\int_{t\in\mathcal{T}_\alpha}\dd{t}
\]
Reporting in Equation~(\ref{eqn:Z-bound}), we get
\[
\mathcal{A}_\alpha U_\alpha^\bot\int_{s\geq1}s^M\exp ms(\log T-\mu_\alpha)\dd{s}
\;\;\leq Z_{[1,\infty)}\leq\;\;
\mathcal{A}_0 U^\top \int_{s\geq1}s^M\exp ms\log T\dd{s}
\]
Hence $Z_{[1,\infty)}$ is finite if $\log T<0$, and infinite if $\log T-\mu_\alpha\geq0$ for at least some $\alpha$, i.e. $\log T>\inf_\alpha\mu_\alpha$. By Lemma~\ref{lem:z-restricted}, the same holds for $Z(m,r)$. Now recall that $\mu_\alpha=D(t)$ where $t$ is one of the corners of $\mathcal{T}_\alpha$. When $\alpha$ tends to $1$, all the corners of $\mathcal{T}_\alpha$ tend to $\tau$, hence $\mu_\alpha$ tends to $D(\tau)=0$. Hence $\inf_\alpha\mu_\alpha=0$.
\end{proof}
\begin{lemma}
\label{lem:m-positive-T-1}
If $r>0$ and $0<m$ and $T=1$, then $Z(m,r)$ is infinite (where $T$ is defined as in the previous lemma).
\end{lemma}
\begin{proof}
Now, $\log T=0$, and, with the notations of the previous lemma, we have, using the change of variable $z=m\mu_\alpha s$:
\[
Z_{[1,\infty)}\geq\mathcal{A}_\alpha U_\alpha^\bot\int_{s\geq1}s^M\exp-m\mu_\alpha s\dd{s}=
\mathcal{A}_\alpha U_\alpha^\bot\frac{\int_{z\geq m\mu_\alpha}z^M\exp-z\dd{z}}{(m\mu_\alpha)^{M+1}}
\]
When $\alpha$ tends to $1$, $\mu_\alpha$ tends to $0$ and the integral in the numerator tends to $\Gamma(M+1)$, while $U_\alpha^\bot$ tends to $U_1^\bot>0$. Hence for some constant $R>0$ we have for $\alpha$ in a neighbourhood of $1$
\[ Z_{[1,\infty)}\geq R\frac{\mathcal{A}_\alpha}{\mu_\alpha^{M+1}} \]
The $k$-th corner of $\mathcal{T}_\alpha$ is given by $t_h=\alpha\tau_h$ for $h\not=k$ and $t_k=\alpha\tau_k+1-\alpha$. It is easy to show, using the fact that $\mu_\alpha$ is the maximum of $D$ on the corners of $\mathcal{T}_\alpha$, that
\[
\mu_\alpha=\max_k\left(\log\alpha+(\alpha\tau_k+1-\alpha)\log(1+\frac{1-\alpha}{\tau_k\alpha})\right)\sim(1-\alpha)^2\max_k\frac{1-\tau_k}{2\tau_k}
\]
Furthermore, using the definition of $\mathcal{A}_\alpha$ and the homeomorphism $t{\mapsto}\tau{+}\frac{t-\tau}{1-\alpha}$ between $\mathcal{T}_\alpha$ and $\mathcal{T}$, we get $\mathcal{A}_\alpha=(1-\alpha)^{K-1}\mathcal{A}_0$. Hence for some constant $R'>0$ we have for $\alpha$ in a neighbourhood of $1$
\[ Z_{[1,\infty)}\geq R'(1-\alpha)^{(K-1)-2(M+1)}= R'(1-\alpha)^{-(K-1)(1+m)-2}\]
When $\alpha$ tends to $1$, the right-hand side tends to $\infty$, hence $Z_{[1,\infty)}$ is infinite, hence so is $Z(m,r)$ by Lemma~\ref{lem:z-restricted}.
\end{proof}
\section{Characterisation of the distribution}
\begin{theorem}
\label{thm:conjugacy}
Let $y_{1:N}$ be a finite family of random variables over $\mathcal{T}$ and $x$ a random variable over $\mathbb{R}_+^K$. We have
\[
\begin{array}{r@{\Rightarrow}l}
\left.
\begin{array}{lr}
\text{Prior:} & x \sim \mbeta(m,r)\\
\text{Observations:} & \{\;y_n|x \sim \mathbf{Dirichlet}(x)\;\}_{n=1}^N\\
\text{Independence:} & \perp\{y_{1:N}\}|x
\end{array}
\right\}
&
\begin{array}{lr}
\text{Posterior:} & x|y_{1:N} \sim \mbeta(m+N,r-\sum_{n=1}^N\log y_n)
\end{array}
\end{array}
\]
This holds whenever the prior is proper, in which case so is the posterior.
\end{theorem}
\begin{proof}
Simple application of the definitions. Distribution $\mbeta$ has been explicitly constructed to be a conjugate prior to the $\mathbf{Dirichlet}$ distribution, which is exactly what Theorem~\ref{thm:conjugacy} states. Obviously, the properness of the prior implies that of the posterior. As a consistency check, we give here a redundant proof of this result, in the case of a single observation $y\in\mathcal{T}$, i.e. $N=1$; the general result ($N$ finite) is simply obtained by reiterating the argument.

Assume the prior $\mbeta(m,r)$ is proper. By Theorem~\ref{thm:properness}, we must first have $r>0$ and $-1<m$. Hence the posterior $\mbeta(m+1,r-\log y)$ satisfies $r-\log y>0$ (because $y\in\mathcal{T}$ hence $\forall k\;y_k\leq1$) and $0<m+1$. By Theorem~\ref{thm:properness}, we must also have $m\leq 0$ or $\sum_k\exp-\frac{r_k}{m}<1$:
\begin{itemize}
\item
If $-1<m\leq0$, then $\frac{1}{m+1}\geq1$, hence $u^{\frac{1}{m+1}}\leq u$ for any $0<u\leq1$ and we have
\[
\sum_k\exp-\frac{r_k-\log y_k}{m+1} = \sum_ky_k^{\frac{1}{m+1}}\exp-\frac{r_k}{m+1} <
\sum_ky_k^{\frac{1}{m+1}}\leq\sum_ky_k=1
\]
\item
If $0<m$ and $\sum_k\exp-\frac{r_k}{m}<1$, using H{\"o}lder's inequality with $p=m+1$ and $q=\frac{m+1}{m}$ we get
\[
\sum_k\exp-\frac{r_k-\log y_k}{m+1} = \sum_ky_k^{\frac{1}{m+1}}\exp-\frac{r_k}{m+1} \leq (\underbrace{\sum_ky_k}_{=1})^{\frac{1}{m+1}}(\underbrace{\sum_k\exp-\frac{r_k}{m}}_{<1})^{\frac{m}{m+1}} <1
\]
\end{itemize}
Hence, by Theorem~\ref{thm:properness}, the posterior $\mbeta(m+1,r-\log y)$ is proper.
\end{proof}
\begin{theorem}
Whenever $\mbeta(m,r)$ is a proper distribution, its moment generating function $\phi(v)$ for $v\in\mathbb{R}^K$, and its $n$-th moment $M_n$ for $n\in\mathbb{N}^K$ are given by
\[
\begin{array}{rcl@{\hspace{2cm}}rcl}
\phi(v) & = & \frac{Z(m,r-v)}{Z(m,r)} &
M_n & = & (-1)^{\sum_kn_k}\frac{1}{Z(m,r)}\frac{\partial^{\sum_kn_k}Z}{\prod_k\partial^{n_k}r_k}(m,r)
\end{array}
\]
\end{theorem}
\begin{proof}
By definition of the moment generating function, we have
\[
\phi(v) = \expectation_{x\sim\mbeta(m,r)}[\exp \inner{v}{x}]=\frac{1}{Z(m,r)}\int_x\mathcal{B}(x)^{-m}\exp(-\inner{r}{x}+\inner{v}{x})\dd{x} = \frac{Z(m,r-v)}{Z(m,r)}
\]
And we have the general result $M_n=\frac{\partial^{\sum_kn_k}\phi}{\prod_k\partial^{n_k}v_k}(0)$.
\end{proof}
In particular, the expectation of $\mbeta(m,r)$ is given by $(M_{\carac{h=k}_{h\in K}})_{k\in K}$, hence
\[
\expectation[\mbeta(m,r)] = -\frac{1}{Z(m,r)}(\frac{\partial Z}{\partial r_k}(m,r))_{k\in K} = -\nabla_r\log Z(m,r)
\]
\section{Numerical computation}
\subsection{Integration over the simplex}
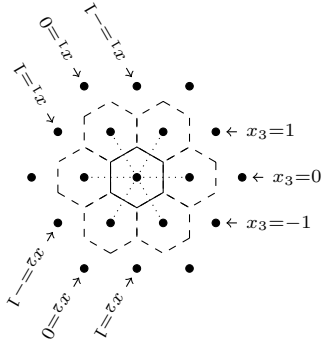
\begin{figure}
\begin{center}
\begin{tikzpicture}[scale=.7]
\begin{scope}[
v/.style={draw,fill,black,circle,inner sep=0cm,minimum size=1mm}
]
\foreach \t in {0,1,2,3,4,5} {
\begin{scope}[rotate=\t*60]
\node[v] at (1,0) {}; \node[v] at (2,0) {}; \node[v] at (1.5,{cos(30)}) {};
\draw[dotted] (0,0) -- (1,0);
\begin{scope}[shift={(1,0)}]
\draw[dashed] (30:{tan(30)}) \foreach \x in {90,150,210,270,330} { -- (\x:{tan(30)}) } -- cycle;
\end{scope}
\end{scope}
}
\node[v] at (0,0) {};
\draw (30:{tan(30)}) \foreach \x in {90,150,210,270,330} { -- (\x:{tan(30)}) } -- cycle;
\foreach \i in {3,1,2} {
\begin{scope}[rotate={\i*120}]
\foreach \t in {1,0,-1} {
\tikzmath{\x=1-cos(60*(3-\t));\y=sin(60*(3-\t));}
\draw [<-] (\x+.2,\y) -- (.4+\x,\y) node[right,rotate={\i*120}] {$\scriptstyle{x_{\i}=\t}$};
}
\end{scope}
}
\end{scope}
\end{tikzpicture}
\end{center}
\caption{\label{fig:grid}Orthogonal view of a sample of the grid $\mathcal{P}^{\mathbb{Z}}_{03}$ in the Euclidian plane $\mathcal{P}_{03}$. Here, the centre point is $000$. Its nearest neighbours in the grid, linked to it by dotted lines, are at distance $\sqrt{2}$. The Voronoi cells of the grid $\mathcal{P}^{\mathbb{Z}}_{0K}$ are regular polytopes of rank $K{-}1$ and measure $1$ (hexagons when $K{=}3$ as above, dodecahedra when $K{=}4$).}
\end{figure}
The family $(\mathcal{P}_{NK})_{N\in\mathbb{Z}}$ of parallel hyperplanes of $\mathbb{R}^K$ is defined as follows:
\[
\mathcal{P}_{NK}\triangleq\{x\in\mathbb{R}^K|\sum_kx_k=N\}
\]
Each hyperplane is equipped with the measure obtained by projection along any axis: it is proportional but not identical to its Lebesgue measure induced by Euclidian distance. It is also the measure used above in all the integrals over the simplex $\mathcal{T}_K$ (notation: $\dd{t}$).

If $A$ is a subset of $\mathbb{R}$, we let $\mathcal{P}^{(A)}_{NK}\triangleq\mathcal{P}_{NK}\cap A^K$. Thus, $\mathcal{P}^{(\mathbb{Z})}_{NK}$ forms a point grid over $\mathcal{P}_{NK}$. For each $x\in\mathcal{P}^{(\mathbb{Z})}_{NK}$, let $\nu_{NK}(x)$ be the Voronoi cell of $x$ with respect to that grid. The family $\{\nu_{NK}(x)|x\in\mathcal{P}^{(\mathbb{Z})}_{NK}\}$ forms a regular tiling of $\mathcal{P}_{NK}$, illustrated in Figure~\ref{fig:grid} when $K{=}3$ (hence $\mathcal{P}_{NK}$ is the 2D plane). For a given $K$, all the cells $\nu_{NK}(x)$ can be obtained from each other by translation, hence their measure depends neither on $x$ nor on $N$, but only on $K$, and is written $\lambda_K$ (in fact, we show below that $\lambda_K=1$). Indeed:
\begin{itemize}
\item
For any pair $x,x'\in\mathcal{P}^{(\mathbb{Z})}_{NK}$, the grid $\mathcal{P}^{(\mathbb{Z})}_{NK}$ is invariant by translation of step $x'-x$, which maps the cell $\nu_{NK}(x)$ into the cell $\nu_{NK}(x')$.
\item 
Similarly, the grid $\mathcal{P}^{(\mathbb{Z})}_{NK}$ is mapped into the grid $\mathcal{P}^{(\mathbb{Z})}_{(N+1)K}$ by translation of step $1$ along any chosen axis. Any point $x{\in}\mathcal{P}^{(\mathbb{Z})}_{NK}$, and its cell $\nu_{NK}(x)$, are mapped by that translation into a point $x'{\in}\mathcal{P}^{(\mathbb{Z})}_{(N+1)K}$, and its cell $\nu_{(N+1)K}(x')$.
\end{itemize}
Now, consider the dilatation of ratio $\frac{1}{N}$ from $\mathcal{P}_{NK}$ into $\mathcal{P}_{1K}$. It is a homeomorphism which transports the grid $\mathcal{P}^{(\mathbb{Z})}_{NK}$ of $\mathcal{P}_{NK}$ into a grid $\frac{1}{N}\mathcal{P}^{(\mathbb{Z})}_{NK}$ of $\mathcal{P}_{1K}$, which contains the grid $\mathcal{P}^{(\mathbb{Z})}_{1K}$ but is much finer. For each $x\in\mathcal{P}^{(\mathbb{Z})}_{NK}$, the cell $\nu_{NK}(x)$ of measure $\lambda_K$ is mapped into a cell $\frac{1}{N}\nu_{NK}(x)$, of measure $\frac{1}{N^{K-1}}\lambda_K$, which tends to $0$ when $N{\rightarrow}\infty$. Hence, if $f$ is any continuous mapping over $\mathcal{P}_{1K}$ with compact support, then the integral of $f$ can be obtained by
\begin{eqnarray}
\label{eqn:grid-approx}
\int_{\mathcal{P}_{1K}}f(t)\dd{t} & = &
\lim_{N\rightarrow\infty}\frac{\lambda_K}{N^{K-1}}\sum_{x\in\mathcal{P}^{(\mathbb{Z})}_{NK}} f(\frac{x}{N})
\end{eqnarray}
\begin{lemma}
\label{lem:voronoi-cell-measure}
$\lambda_K=1$ for all $K\geq1$.
\end{lemma}
\begin{proof}
For any $N\geq0$ and $K\geq1$, let $q_{NK}\triangleq|\mathcal{P}^{(\mathbb{N})}_{NK}|$. The sequences of $\mathbb{N}^{K+1}$ summing to $N$ are in bijection with the sequences of $\mathbb{N}^K$ summing to a number between $0$ and $N$. Hence $q_{N(K+1)}=\sum_{m=0}^Nq_{mK}$. And, by definition, $q_{N1}{=}1$.

Let $F=\sum_{N\geq0,K\geq1}q_{NK}X^NY^K$ be the characteristic function of that sequence. We have, using the recurrence,
\begin{eqnarray*}
F & = &
\sum_NX^NY + \sum_{NK}\sum_{m=0}^Nq_{mK}X^NY^{K+1} =
\sum_NX^NY + \sum_{mK}q_{mK}(\sum_{N\geq m}X^N)Y^{K+1}\\
& = &
\frac{1}{1-X}Y + \sum_{mK}\frac{X^m}{1-X}q_{mK}Y^{K+1} = 
\frac{Y}{1-X}(1+F)\;\;\text{hence}\\
F & = & \frac{Y}{1-X-Y}=
Y\sum_m(X+Y)^m = Y\sum_{NK}\binom{N+K-1}{K-1}X^NY^{K-1} =
\sum_{NK}\binom{N+K-1}{K-1}X^NY^K
\end{eqnarray*}
Hence $q_{NK}=\binom{N+K-1}{K-1}$ and it is easy to see that $q_{NK}\sim\frac{N^{K-1}}{(K-1)!}$ when $N{\rightarrow}\infty$. Applying Equation~(\ref{eqn:grid-approx}) to the indicator function of the simplex $\mathcal{T}_K$, we get
\[
\int_{\mathcal{T}}\dd{t}=\lim_{N\rightarrow\infty}\frac{\lambda_K}{N^{K-1}}q_{NK}=\frac{\lambda_K}{(K-1)!}
\]
On the other hand, we have $\int_{\mathcal{T}}\dd{t}=\mathcal{B}(1)=\frac{\Gamma(1)^K}{\Gamma(K)}=\frac{1}{(K-1)!}$, hence $\lambda_K=1$.
\end{proof}
We now consider Equation~(\ref{eqn:grid-approx}) in the case where $f$ is null outside $\mathcal{T}_K$, and is factorised, i.e. there exists a family $f_{1:K}$ of scalar functions $f_k:\mathbb{R}_+\mapsto\mathbb{R}$ such that
\[
f(t) = \prod_{k=1}^Kf_k(t_k)
\]
In that case, for a given $N$, the sum in the right-hand side of Equation~(\ref{eqn:grid-approx}) becomes
\begin{eqnarray}
\label{eqn:sum-over-simplex}
S_{NK}(f_{1:K}) & \triangleq & \sum_{x\in\mathcal{P}^{(\mathbb{N})}_{NK}}\prod_{k=1}^Kf_k(\frac{x_k}{N})
\end{eqnarray}
By construction, if $\otimes$ denotes the convolution operator on vectors indexed by $0{:}N$, then
\[
S_{NK}(f_{1:K}) = \left(\bigotimes_{k=1}^K\left(f_k(\frac{n}{N})\right)_{n=0:N}\right)_N
\]
This convolution can be computed efficiently without exhaustively enumerating the grid $\mathcal{P}^{\mathbb{N}}_{NK}$. Depending on the magnitude of $N$, working with Fast Fourier Transforms may even be more efficient. In the log domain, a simple method makes use of the $\log\otimes\exp$ operator, which is associative commutative like $\otimes$, and which can be computed efficiently for any $0{:}N$-dimensional vectors $x,y$ by:
\[
\log\otimes\exp(x,y) = \log\sum{_{.*}}\exp(T(x)\dot{+}y)
\]
where $T(x)$ is the $0{:}N{\times}0{:}N$ Toeplitz matrix where $T(x)_{nm}$ is equal to  $x_{n-m}$ if $n\geq m$ and $-\infty$ otherwise, operator $\dot{+}$ adds a vector to a matrix row-wise and returns a matrix, and operator $\log\sum_{.*}\exp$ applies operator $\log\sum\exp$ to each row of a matrix and returns a vector. Operator $\log\sum\exp$ must be efficiently implemented to avoid numerical instability.
\subsection{Approximate computation of $Z(m,r)$}
For a given $s\in\mathbb{R}_+$, the integrand in the definition of $\bar{Z}(s)$ is factorised (up to a multiplicative constant):
\begin{eqnarray*}
\label{eqn:z-bar-approx}
\bar{Z}(s) & = &
\Gamma(s)^m\int_{\mathcal{T}}\prod_{k=1}^K\exp(-m\log\Gamma(st_k)+r_kst_k)\dd{t}
\end{eqnarray*}
Thus, it can be approximated using Equations~(\ref{eqn:grid-approx}) and~(\ref{eqn:sum-over-simplex}) where $f_k(u;s)=\exp(-m\log\Gamma(su)+r_ksu)$. For $N$ sufficiently large,
\begin{eqnarray}
\bar{Z}(s) & \approx &
\frac{\Gamma(s)^m}{N^{K-1}}S_{NK}(f_{1:K}(.;s))
\end{eqnarray}
Now, choose any pivot value $\rho\in\mathbb{R}_+$. Let $\text{Gam}(K,\rho)$ be the Gamma distribution with shape $K$ and rate $\rho$. we have:
\[
Z(m,r) =
\int_s\bar{Z}(s)s^{K-1}\dd{s} = \int_s\bar{Z}(s)\exp(s\rho)s^{K-1}\exp-s\rho\dd{s} =
\frac{\Gamma(K)}{\rho^K}\expectation_{s\sim\text{Gam}(K,\rho)}[\bar{Z}(s)\exp s\rho]
\]
The expectation can be approximated by the average of a sample $(s_p)_{p=1:P}$ from the distribution $\text{Gam}(K,\rho)$ for $P$ sufficiently large. Combining with Equation~(\ref{eqn:z-bar-approx}), we get:
\begin{eqnarray}
Z(m,r) & \approx &
\frac{\Gamma(K)}{P\rho^KN^{K-1}}\sum_{p=1}^PS_{NK}(f_{1:K}(.;s_p))\Gamma(s_p)^m\exp s_p\rho
\end{eqnarray}
In log scale, this becomes
\[
\log Z(m,r) \approx
\log\Gamma(K)-\log P-K\log\rho-(K-1)\log N+\log\sum_p\exp(\log S_{NK}(f_{1:K}(.;s_p))+m\log\Gamma(s_p)+\rho s_p)
\]
Moreover, for $p\in1{:}P$, let $D_p$ be the $0{:}N$ vector $(s_p\frac{n}{N})_{n=0:N}$. Then
\[
\log S_{NK}(f_{1:K}(.;s_p)) =
\left(\log\bigotimes_{k=1}^K\exp(-m\log\Gamma(D_p)+r_kD_p)\right)_N
\]
\printbibliography
\end{document}